\documentclass[hidelinks,onefignum,onetabnum,final]{siamart220329}

\usepackage{lipsum}
\usepackage{amsfonts}
\usepackage{graphicx}
\usepackage{epstopdf}
\ifpdf
  \DeclareGraphicsExtensions{.eps,.pdf,.png,.jpg}
\else
  \DeclareGraphicsExtensions{.eps}
\fi

\newsiamremark{remark}{Remark}
\newsiamremark{hypothesis}{Hypothesis}
\crefname{hypothesis}{Hypothesis}{Hypotheses}
\newsiamthm{claim}{Claim}

\headers{Convergence of RS-CAVI under log-concavity}{H. Lavenant and G. Zanella}

\title{Convergence rate of random scan Coordinate Ascent Variational Inference under log-concavity\thanks{\textbf{Funding:} HL acknowledges the support of the MUR-Prin 2022-202244A7YL funded by the European Union- Next Generation EU. GZ acknowledges support from the European Research Council (ERC), through StG ``PrSc-HDBayLe''
grant ID 101076564.}}

\author{Hugo Lavenant\thanks{Bocconi University, Department of Decision Sciences and BIDSA, Milan, Italy 
  (\email{hugo.lavenant@unibocconi.it}).}
\and Giacomo Zanella\thanks{Bocconi University, Department of Decision Sciences and BIDSA, Milan, Italy 
  \;(\email{giacomo.zanella@unibocconi.it)}}}

\usepackage{amsopn}

\ifpdf
\hypersetup{
  pdftitle={Convergence rate of RS-CAVI under log-concavity},
  pdfauthor={H. Lavenant and G. Zanella}
}
\fi

\usepackage{algorithm}
\usepackage{algpseudocode}

\newsiamremark{asmp}{Assumption}

\newcommand{\R}{\mathbb{R}}

\newcommand{\sX}{\mathcal{X}}
\newcommand{\sP}{\mathcal{P}}

\begin{document}

\maketitle

\begin{abstract}
The Coordinate Ascent Variational Inference scheme is a popular algorithm used to compute the mean-field approximation of a probability distribution of interest. 
We analyze its random scan version, under log-concavity assumptions on the target density.
Our approach builds on the recent work of M. Arnese and D. Lacker, \emph{Convergence of coordinate ascent variational inference for log-concave measures via optimal transport} [arXiv:2404.08792] which studies the deterministic scan version of the algorithm, phrasing it as a block-coordinate descent algorithm in the space of probability distributions endowed with the geometry of optimal transport. 
We obtain tight rates for the random scan version, which imply that the total number of factor updates required to converge scales linearly with the condition number and the number of blocks of the target distribution.
By contrast, available bounds for the deterministic scan case scale quadratically in the same quantities, which is analogue to what happens for optimization of convex functions in Euclidean spaces. 
\end{abstract}

\begin{keywords}
Mean field approximation, 
Block coordinate ascent, 
Convex optimization, 
Optimal transport, 
Bayesian computation, Iteration complexity.
\end{keywords}


\section{Introduction}
Let $\sX=\sX_1\times\dots\times \sX_K$ be a product space where each $\sX_k$ is a closed  convex subset of $\R^{d_k}$. This includes the case $\sX = \R^d$ with $d=d_1 + \ldots +d_K$. 
Let $\sP(\sX)$ denote the space of probability distributions over $\sX$, and $\pi\in\sP(\sX)$ be a given distribution of interest. 
Define the mean-field (MF) approximation to $\pi$ as
\begin{equation}\label{eq:MF_def}
q^*=\arg\min_{q\in\sP^{\otimes K}(\sX)} \mathrm{KL}(q \| \pi) \,,
\end{equation}
where $\sP^{\otimes K}(\sX)=\sP(\sX_1)\times\dots\times \sP(\sX_K)$ denotes the space of factorized probability distributions (a.k.a. product distributions) over $\sX$ and 
$\mathrm{KL}(q \| \pi)=\int \log(dq/d\pi) dq$ denotes the Kullback-Leibler (KL)-divergence between $q$ and $\pi$. Note that we use $q \in \sP^{\otimes K}(\sX)$ to denote both the $K$-uplet of factors $(q_1, \ldots, q_K) \in \sP(\sX_1)\times\dots\times \sP(\sX_K)$ but also the factorized probability distribution $q=q_1 \otimes \ldots \otimes q_K$ on $\mathcal{P}(\sX)$, that is, the product of the $K$ factors. 
In the sequel we take $\pi$ as fixed and abbreviate $F(q) = \mathrm{KL}(q \| \pi)$.
We assume that $\pi$ has a density with respect to the Lebesgue measure of the form $\exp(-U(x))/Z$ for some $U:\sX\to\R$ and $Z>0$. 
While the Mean-Field approximation $q^*$ is not always unique, it is so if $U$ is strictly convex \cite{AL24}. 
Mean-field approximations are commonly used to approximate complex and potentially intractable distributions arising from probabilistic modeling in various scientific fields, including Bayesian Statistics, Statistical Physics and Machine Learning \cite{J+99,B06,WJ08,B+17}.

In many contexts the minimizer $q^*$ is computed using the Coordinate Ascent Variational Inference (CAVI) algorithm, which iteratively minimizes $F$ with respect to one element of $(q_1, \ldots, q_K)$ at a time \cite{B+17}. 
Depending on the order in which elements are updated, one obtains different variants of CAVI. 
Here we consider the random scan (RS) version, described in Algorithm~\ref{alg:rs-cavi}. \newpage

 \begin{algorithm}[htbp]
   \begin{algorithmic}
\State Input: target distribution $\pi\in\sP(\sX)$; Initialization $q^{(0)}=(q^{(0)}_1,\dots,q^{(0)}_K)\in\sP^{\otimes K}(\sX)$
\For{$n=0,1,2,\dots,$}
     \State Sample $k_n\sim \hbox{Unif}(\{1,\dots,K\})$
  \State Set
\begin{align}\label{eq:cavi_alg_update}
    q^{(n+1)}_{k_n}\gets \arg\min_{q_{k_n}\in\sP(\sX_{k_n})}
    F(q^{(n)}_1,\dots,q^{(n)}_{k_n-1},q_{k_n},q^{(n)}_{k_n+1},\dots,q^{(n)}_{K})
\end{align}
  \State Set $q^{(n+1)}_i\gets q^{(n)}_i$ for $i\neq k_n$
   \EndFor
\end{algorithmic}
\caption{(RS-CAVI)
\label{alg:rs-cavi}}
 \end{algorithm}
 
The $\arg\min$ over $\mathcal{P}(\sX_k)$ involved in \eqref{eq:cavi_alg_update} is unique by the strict convexity of $F$, and the update rule~\eqref{eq:cavi_alg_update} takes the form
\begin{equation}\label{eq:CAVI_updates}
q^{(n+1)}_{k_n}= \exp(-U_{k_n}^{(n+1)})/Z_{k_n}^{(n+1)}
\,,\qquad
U^{(n+1)}_{k_n} = \int_{\sX_{-k}} U d q^{(n)}_{-k_n}\,,
\end{equation}
with $Z_{k_n}^{(n+1)} = \int \exp(-U_{k_n}^{(n+1)}(x_{k_n}))dx_{k_n}$. Throughout the paper,  $x_{-k}=(x_{i})_{i\neq k}$ denotes the vector $x$ without its $k$-th coordinate and $\sX_{-k}=\times_{i\neq k}\sX_{i}$ the corresponding state space. 
In typical applications of CAVI, the integrals in \eqref{eq:CAVI_updates} can be computed in closed form, so that Algorithm \ref{alg:rs-cavi} -- as well as its cyclical and parallel update variants -- can be implemented efficiently without resorting to numerical integration. See e.g.\ \cite{B06,B+17,BPY23} and references therein for more details on applications of MF-VI, as well as examples of CAVI implementations.

\paragraph{Convergence analysis of CAVI}

In this work we analyze the speed at which the iterates $q^{(n)}$ produced by Algorithm \ref{alg:rs-cavi} converge to $q^*$ as $n \to + \infty$, see Theorem~\ref{thm:RS_CAVI_contraction} for our main result. 
We view CAVI as a block-coordinate descent algorithm, an optimization procedure which is well understood when the objective function $F$ is defined on an Euclidean space and convex \cite{W15}. 
However in CAVI the function $F$ is defined on the space of probability distributions and, more crucially, we seek to minimize it over $\sP^{\otimes K}(\sX)$, which is not a convex subset of $\mathcal{P}(\sX)$  since a mixture of product distributions is not a product distribution in general. Thus, even though $F$ is convex with respect to usual linear structure of probability distributions, the problem~\eqref{eq:MF_def} is not convex with respect to that structure. Such non-convexity has been widely discussed in the literature \cite{WJ08}, and has drastically limited the amount of convergence results available for CAVI until recently \cite{BPY23}. 
The key insight of the recent work by Arnese and Lacker~\cite{AL24}, building on previous work in \cite{L23} and on the broad literature on the analysis of sampling algorithms under log-concavity \cite{C24}, 
is to endow $\mathcal{P}(\sX)$ with the geometry coming from optimal transport, a.k.a. the Otto, or Wasserstein geometry. 
This has two main advantages: the function $F = \mathrm{KL}(\cdot \| \pi)$ is (geodesically) convex in this geometry with convexity parameters depending on the ones of the negative log density $U$; and the subspace $\sP^{\otimes K}(\sX)$ of product distributions is geodesically convex in $\mathcal{P}(\sX)$. 
In addition, regardless of the geometry under consideration, the objective function $F$ decomposes as, for any $q=(q_1,\dots,q_K)\in\sP^{\otimes K}(\sX)$,
\begin{equation}
\label{eq:decomposition_F_composite}
F(q)=f(q)
+
\sum_{k=1}^KH(q_k)
\end{equation}
where
$H( q_k)=\int_{\sX_k}\log q_k dq_k$
is the entropy functional on $\sP(\sX_k)$ and 
$f(q)=\int_{\sX}U dq$. This leads to a \emph{composite} optimization setting: it can be analyzed with assumptions of convexity and smoothness on $f$, which depends jointly on all factors $q_1, \ldots, q_K$, while the separable part $\sum_k H(q_k)$ is only assumed to be convex.

The work~\cite{AL24} applied to these two key insights -- using the Wasserstein geometry and the decomposition~\eqref{eq:decomposition_F_composite} -- to analyze the deterministic scan version of CAVI, that is, when $q_k$ is updated cyclically for $k=1,\dots,K$ at each algorithmic iteration. In the present work we consider the random scan case, taking inspiration from the proof of convergence of RS block-coordinate descents algorithms in the Euclidean case~\cite{RT14,W15}. 
The random scan assumption makes the theoretical analysis actually simpler, and leads to rates that are significantly faster than for the deterministic scan version, a feature already present in optimization over Euclidean spaces~\cite{W15,L+18}. 
Also, our rate for the case of strongly log-concave target distributions is tight, matching lower bounds available for the case of Gaussian target distributions.

We emphasize that we tried to keep the proofs as elementary as possible and we avoid defining the tangent space in the space of probability distributions, Wasserstein subdifferentials, Wasserstein gradients, etc. The only deep result of optimal transport we use is the geodesic convexity of the entropy (see~\eqref{eq:convexity_entropy}); all the other tools are elementary elements of convex analysis. We believe that this can make the proof ultimately clearer and more accessible. Interestingly, our key intermediate results (see Lemma~\ref{lm:bound_exp_with_H} and Lemma~\ref{lm:H_contraction}) are purely phrased in a metric language: they only involve function value and optimal transport distances without reference to gradient or tangent spaces, thus it is reasonable to expect that they can be extended to geodesically convex functions on general metric spaces.

\paragraph{Related works on the MF-VI problem}

Beyond~\cite{AL24}, we also mention the recent works~\cite{BPY23}, which relate the convergence properties of the two-block version of CAVI to a measure of generalized correlation, and~\cite{YY22,JSP23,YCY24}, which analyze algorithms, different from CAVI, to compute the mean-field approximation $q^*$ under log-concavity assumptions. While~\cite{YY22,JSP23,YCY24} still rely on the Wasserstein geometry, they do not assume to have access to a closed form expression for~\eqref{eq:CAVI_updates}, thus having to introduce additional layers of approximation (e.g. numerical integration or discretization) to handle the update rules of their algorithms. This arguably makes a direct comparison of their results (e.g.\ algorithmic convergence rates) with ours or the ones in~\cite{AL24} not particularly meaningful or insightful. 

\bigskip 

\noindent The main results (Theorem~\ref{thm:RS_CAVI_contraction} and Corollary~\ref{cor:num_iters}) are stated and commented in Section~\ref{sec:main}. The remaining sections are dedicated to background results (Section~\ref{sec:background}) and the proof of these main results (Section~\ref{sec:proof}).

\section{Main results}
\label{sec:main}

We make the following convexity and smoothness assumptions on $U$, the negative log density of the target distribution $\pi$. We denote by $\nabla_k U : \sX \to \R^{d_k}$ the gradient of $U$ with respect to the variable $x_k$. 

\begin{asmp}\label{ass:conv_smooth_potential}
The partial derivatives of $U:\sX\to\R$ exist and $U$ is block coordinate-wise smooth with smoothness constants $(L_1,\dots,L_K)$, which means that for every $k\in\{1,\dots,K\}$ and $x_{-k}\in\sX_{-k}$, the function $x_k:\sX_k\mapsto\nabla_k U(x)$ is $L_k$-Lipschitz. 
In addition, for some $\lambda^* \geq 0$, the function $U$ is $\lambda^*$-convex under the $\|\cdot\|_L$ metric defined as $\|x\|^2_L=\sum_{k=1}^K L_k\|x_k\|^2$, which means that the function $x \mapsto U(x) - \frac{\lambda^*}{2} \| x \|_L^2$ is convex. Here $\lambda^* \geq 0$ by assumption, $\lambda^* \leq 1$ by construction, and $\lambda^*>0$ corresponds to the strongly convex case.
\end{asmp}

The convexity parameter $\lambda^*$ coincides with the convexity parameter of $U$ after rescaling each coordinate in such a way that $L_k=1$ for all $k$ (i.e.\ defining new coordinates $\tilde{x}=(\tilde{x}_1,\dots,\tilde{x}_K)$ as $\tilde{x}_k=L_k^{-1/2}x_k$ for all $k$). 
Also, since CAVI is invariant with respect to change of coordinates (provided the new $\tilde{x}_k$ is only a function of $x_k$), one could actually apply any invertible transformation to each space $\sX_k$ so that Assumption~\ref{ass:conv_smooth_potential} is satisfied in the new coordinates.

If the function $U$ satisfies the more usual assumptions of $L$-smoothness and $\lambda$-convexity (that is, $\nabla U$ is $L$-Lipschitz while $x \mapsto U(x) - \frac{\lambda}{2} \| x \|^2$ is convex for $\lambda \geq 0$) then it satisfies Assumption~\ref{ass:conv_smooth_potential} with $\lambda^*\geq \lambda/L$. Indeed, writing $L_\text{max}=\max_{k=1,\dots,K}L_k$, it holds  $L_\text{max}\leq L\leq K L_\text{max}$ (see e.g.\ \cite[Lemma 1]{N12}) and can be checked easily that $\lambda^* \geq \lambda / L_\text{max} \geq \lambda / L$. 
Indeed, one can interpret $1/\lambda^*$ as a ``coordinate-wise" condition number, which measures dependence across coordinates under $\pi$, and it is always smaller than the usual condition number $L/\lambda$. To illustrate the difference, consider the case of $U(x) = \sum_{k=1}^K \lambda_k \| x_k \|^2$ so that $\pi$ is a Gaussian measure with independent components. Then $L_k = \lambda_k$ for all $k$ and $\lambda^* = 1$, while on the other hand the condition number $L/\lambda = \max_k \lambda_k / \min_k \lambda_k$ can be arbitrary large. In this case $\pi$ coincides with its mean-field approximation and the CAVI algorithm converges after each factor has been updated at least once, which takes of the order of $K$ iterations (up to logarithmic terms), independently on the values of the $\lambda_k$'s.

\begin{remark}[Checking the assumption for $C^2$ potentials]
If $U$ is of class $C^2$, $\lambda$ and $L$, the global convexity and smoothness constants of $U$, are the uniform lower and upper bound on the eigenvalues of $\nabla^2 U$ the Hessian matrix of $U$. Assumption~\ref{ass:conv_smooth_potential} can also be checked by looking at $\nabla^2 U$. Block smoothness means that the $d_k \times d_k$ sub-block of $\nabla^2 U(x)$, i.e.\ its $k$-th diagonal sub-block, should be smaller than $L_k \mathrm{Id}_{d_k}$. Moreover, denoting by $D_L$ the $d\times d$ block-diagonal matrix with diagonal blocks $(L_k \mathrm{Id}_{d_k})_{k=1}^K$, Assumption \ref{ass:conv_smooth_potential} implies that, for a.e. $x$, the eigenvalues of $D_L^{-1/2} \nabla^2 U(x) D_L^{-1/2}$ are all larger than $\lambda^*$.
\end{remark}

We endow $\sP^{\otimes K}(\sX)$ with the Wasserstein distance $W_{2,L}^2(q,\tilde q) = \sum_{k=1}^K L_k W^2_2(q_k, \tilde q_k)$ between measures having finite second moments, where $W_2$ is the classical Wasserstein distance with squared euclidean cost, see Section~\ref{sec:background} for more details.
We are now ready to state our main theorem. Note that below the randomness comes only from the choice of the updated factor $k_n$ at each iteration $n$, and both expectation and probability, denoted by $\mathbb{E}$ and $\mathbb{P}$, are understood with respect to this randomness. 

\begin{theorem}\label{thm:RS_CAVI_contraction}
Let $(q^{(n)})_{n=0,1,2,\dots}$ be the sequence induced by Algorithm \ref{alg:rs-cavi} and $q^*$ a minimizer of \eqref{eq:MF_def}.
    Under Assumption \ref{ass:conv_smooth_potential}, for every $n\geq 0$, we have:
    \begin{enumerate}
        \item[(a)] If $\lambda^*>0$ then
        \begin{equation}\label{eq:RS_CAVI_contraction}
\mathbb{E}[F( q^{(n)})]-F(q^*)
\leq
\left(1-\frac{\lambda^*}{K}\right)^n
(F( q^{(0)})-F(q^*))
\,.
\end{equation} 
        \item[(b)] If $\lambda^*=0$ then
\begin{equation}
\label{eq:RS_CAVI_linear}
\mathbb{E}[F( q^{(n)})]-F(q^*)
\leq
\frac{2KR^2}{n+2K} 
\,,
\end{equation} 
where $R=\max \left\{\sqrt{F(q^{(0)}) - F(q^*)},\sup_{n=0,1,2,\dots}W_{2,L}(q^{(n)},q^*) \right\}<\infty$.
    \end{enumerate}
\end{theorem}
By Markov's inequality, the upper bounds in Theorem \ref{thm:RS_CAVI_contraction} directly imply the following high-probability bounds on the number of iterations required by RS-CAVI to produce an approximation with a given suboptimality gap. 
\begin{corollary}\label{cor:num_iters}
Under the same notation of Theorem \ref{thm:RS_CAVI_contraction}, assume that either $\lambda^*>0$ and
    $$
    n\geq \frac{K}{\lambda^*}\log\left(\frac{F( q^{(0)})-F(q^*)}{\epsilon\delta}\right),
    $$
or $\lambda^*=0$ and
    $$
    n\geq 2K \left( \frac{R^2}{\epsilon \delta} - 2 \right)\,,
    $$
for some $\epsilon, \delta>0$. Then
    $$\mathbb{P}(F( q^{(n)})-F(q^*)<\epsilon)\geq 1-\delta\,.$$
\end{corollary}

Note that we are implicitly assuming $F(q^{(0)})<\infty$ since otherwise the statements are uninformative. 
It is always possible to initialize with $F(q^{(0)}) < + \infty$:  we could take $q^{(-1)}=\delta_{x}$ for some $x\in\sX$ and then obtaining the $K$ factors $(q^{(0)}_k)_{k=1}^K$ by performing $K$ updates as in \eqref{eq:CAVI_updates} cyclically; see e.g.\ \cite[Theorem 1]{AL24}. 

\begin{remark}[Bounds in Wasserstein distance]
In the case $\lambda^* > 0$ the optimality gap $F(q) - F(q^*)$ controls the Wasserstein distance to the mean-Field approximation (see~\cite[Proposition 6.3]{AL24}):
\begin{equation}
\label{eq:control_F_W2}
\frac{\lambda^*}{2} W_{2,L}^2(q,q^*) \leq F(q) - F(q^*).
\end{equation}
Thus our result also implies a quantitative rate of convergence in Wasserstein distance.    
\end{remark}

\begin{remark}[Tightness of the rate]
The best improvement of the geometric rate $\left(1-\lambda^*/K\right)$ in \eqref{eq:RS_CAVI_contraction} may be to a geometric rate $\left(1-\lambda^*/K\right)^2$, which implies that our bound is tight up to a factor $2$ in terms of number of iterations required for convergence. Indeed, if $\sX=\R^d$ and $U(x)=\frac{1}{2}x^\top Qx + b^\top x$ with $Q$ symmetric positive definite and $b\in\R^d$, then by \cite[Theorem 1]{GPZ23} we have
    $$
\mathbb{E}[F( q^{(n)})]-F(q^*)
\geq
\left(1-\frac{\lambda_\text{min}(D_Q^{-1/2}QD_Q^{-1/2})}{K}\right)^{2n}
(F( q^{(0)})-F(q^*))\,,
$$
for at least one $q^{(0)}\in \sP^{\otimes K}(\sX)$. Here $\lambda_\text{min}$ denotes the minimum eigenvalue and $D_Q$ the $d\times d$ block-diagonal matrix with diagonal blocks $(Q_{kk})_{k=1}^K$ equal to the ones of $Q$.
Note that $\lambda_\text{min}(D_Q^{-1/2}QD_Q^{-1/2})$ is the convexity constant of $U$ under the metric $\|x\|_{D_Q}^2=\sum_{k=1}^K x_k^\top Q_{kk}x_k$, meaning that it coincides with $\lambda^*$ when $d_k=1$ for all $k$ and more generally it coincides with $\lambda^*$ after reparametrizing the state space $\sX$ by $D_Q^{-1/2}$.

To the best of our knowledge, the best proved geometric rate for RS coordinate descent in the Euclidean case is 
$(1 - 2 \lambda^* / (K(1+\lambda^*)))$,
see \cite[Theorem 1]{LX15}.
\end{remark}

\begin{remark}[Comparison with deterministic scan CAVI]
The work \cite{AL24} provides bounds on the number of iterations required by the cyclical, also called deterministic scan (DS), version of CAVI where at each algorithmic iteration each factor $q_k$ going through $k=1,\dots,K$ is updated.
The computational cost of one DS-CAVI iteration is equal to $K$ times the one of RS-CAVI, since $m K$ factor updates are required to obtain  perform $m$ iterations of DS-CAVI.
 Assuming $U$ to be $\lambda$-convex and $L$-smooth \cite[Theorem 1.1(d)]{AL24} implies that
 \begin{align*}
 m K \geq 
 K\frac{L^2K+\lambda^2}{\lambda^2}\log\left(\frac{F( q^{(0)})-F(q^*)}{\epsilon}\right)    
 \geq
 \frac{K^2}{(\lambda^*)^2}\log\left(\frac{F( q^{(0)})-F(q^*)}{\epsilon}\right)    
 \end{align*}
factor updates are sufficient to guarantee $F( q^{(m)}_{DS})-F(q^*)<\epsilon$, where $q^{(m)}_{DS}$ denotes the density obtained after $m$ iterations of DS-CAVI.
In terms of dependence with respect to $K$ and $\lambda^*$ this number of updates is a factor of $K/\lambda^*$ larger than the one in Corollary \ref{cor:num_iters}.
This is analogous to the well-known fact that, in the Euclidean case, cyclical coordinate ascent is $K$ times worse than its randomized version in terms of worst-case computational complexity, see e.g.\ discussions in \cite{W15,L+18} and references therein. Similar considerations hold for the convex case $(\lambda^*=0)$, where the bounds in \cite[Theorem 1.1(c)]{AL24} imply a $\mathcal{O}\left(K^2R^2/\epsilon\right)$ bound on the number of updates required by DS-CAVI to reach a suboptimality gap of $\varepsilon$.
\end{remark}

\begin{remark}[Control on $R$]
The bound~\eqref{eq:RS_CAVI_linear} for $\lambda^*=0$ involves the quantity $\sup_{n=0,1,2,\dots}W_{2,L}(q^{(n)},q^*)$.
Analogous terms appear in the Euclidean context, see e.g.\ \cite{RT14}, and they are usually bounded by the diameter of sublevel sets of the target function $F$. When $U$ is strongly convex, or more generally satisfies a Talagrand inequality, then \eqref{eq:control_F_W2} holds and thus sublevel sets of $F$ have finite diameter. However, if $U$ is not strongly convex, sublevel sets of $F$ may contain distributions with infinite second moments and need not be bounded for the distance $W_{2, L}$. Nonetheless as the $q^{(n)}$'s become log-concave distributions it is possible to bound their second moment. We report the argument in Proposition~\ref{prop:R_bounded} and Corollary~\ref{corollary:bound_distance_entropy} in the appendix which improve the bounds proved in~\cite{AL24} for the deterministic scan case.
It is still unclear if this bound is tight in terms of dependence in $F(q^{(0)})$ the initial Kullback-Leibler divergence.
\end{remark}

\section{Background results and notation}
\label{sec:background}

We recall some useful results of convex analysis and optimal transport which we will use in the proof. The former is standard, we refer to~\cite{W15} and references therein. Regarding the latter, we refer the readers to the textbooks~\cite{AGS2008,Villani2009} for a comprehensive presentation.

In the sequel, for a function $\psi(y_1, \ldots, y_K)$ depending on $K$ variables, by an abuse of notations we denote by $\psi(z_k, y_{-k})$ the value of the function $\psi$ where the $k$-th component of $y$ has been replaced by $z_k$ and the other ones are left unchanged. We will use it with $\psi = U$ defined on $\sX$, and for $\psi = f, F$ or $H$ defined on $\mathcal{P}^{\otimes K}(\sX)$.

\paragraph{Convexity and smoothness}

We recall here the fundamental inequalities implied by the assumptions of smoothness of convexity of the potential $U$. Note that we will only use them to derive similar inequalities for the function $f : q \mapsto \int U d q$ below. Let $U$ satisfy Assumption~\ref{ass:conv_smooth_potential} and fix $x,y \in \sX$. Then we have the following lower bound on the detachment of $U$ from its linear approximation:
\begin{equation}
\label{eq:convexity_U_gradient}
U(y) \geq U(x) + \nabla U(x)^\top (y - x) + \frac{\lambda^*}{2} \| y - x \|_L^2.
\end{equation}
On the other hand, fixing $k \in \{1, \ldots, K \}$, the assumption of smoothness for the $k$-th block easily implies the following upper bound on the detachment of $U$ from its linear approximation in the $k$-th block:
\begin{equation}
\label{eq:smoothness_U}
U(y_k, x_{-k}) \leq U(x) +\nabla_k U(x)^\top (y_k - x_k) + \frac{L_k}{2} \| y_k - x_k \|^2.
\end{equation}
Finally, we have the following inequality implied by convexity which does not involve gradients: for $t \in [0,1]$,
\begin{equation}
\label{eq:asmp_convex_u_classical}
U((1-t)x + ty) \leq (1-t)U(x) + t U(y) - \frac{\lambda^*}{2} t(1-t) \| y - x \|^2_{L}. 
\end{equation}

\paragraph{Optimal transport on a factor}

Recall that $\sX_k$ is a closed convex subset of $\mathbb{R}^{d_k}$, possibly $\mathbb{R}^{d_k}$ itself. We denote by $\mathcal{P}_2(\sX_k)$ the space of probability distributions with bounded second moments on $\sX_k$, and we endow it with the quadratic Wasserstein distance $W_2$. At least for $q, \tilde{q} \in \mathcal{P}_2(\sX_k)$ having a density with respect to the Lebesgue measure, it reads
\begin{equation*}
W_2^2(q,\tilde{q}) = \min \left\{ \int_{\sX_k} \| T(x) - x \|^2 q(dx) \,;\; T : \sX_k \to \sX_k\text{ such that } \ T_\# q = \tilde{q} \right\},
\end{equation*}
where $T_\#$ denotes the push-forward under a map $T$. It can be proved that $W_2$ defines a metric which metrizes weak convergence together with convergence of second moments~\cite[Chapter 7]{AGS2008}. Still in the setting where $q$ has a density with respect to the Lebesgue measure, the map $T$ realizing the infimum exists and is unique $q$-a.e.~\cite[Chapter 6]{AGS2008}. We call it the optimal transport map between $q$ and $\tilde q$ and denote it $T_{ q}^{\tilde{q}}$.  

We define the \emph{geodesic} $(q_t)_{t \in [0,1]}$ between $q$ and $\tilde q$ as the curve $t \mapsto ((1-t)\mathrm{Id} + t T^{\tilde q}_q)_\# q$, valued in $\mathcal{P}_2(\sX_k)$, see~\cite[Section 7.2]{AGS2008}. It satisfies $q_0 = q$, $q_1 = \tilde q$ and in addition for any $t,s\in[0,1]$
\begin{equation}
\label{eq:def_geo_factor}
W_2(q_t,q_s) = |t-s| W_2(q,\tilde q).
\end{equation}
A fundamental result of optimal transport theory is that the differential entropy function $H$, which we recall is $q \mapsto \int q \log q$ on $\mathcal{P}(\sX_k)$, is convex along geodesics. We refer to~\cite[Section 9]{AGS2008} or \cite[Chapter 16 \& 17]{Villani2009} for a detailed account and historical references. Below we will only use the following fact~\cite[Proposition 9.3.9]{AGS2008}: for any $t \in [0,1]$ and any $q$, $\tilde q$ in $\mathcal{P}(\sX_k)$,
\begin{equation}
\label{eq:convexity_entropy}
H(q_t) \leq (1-t) H(q) + t H(\tilde q).
\end{equation}

\paragraph{Optimal transport on factorized distributions}

The structure above extends very well to the product space. If $q, \tilde{q} \in \sP_2^{\otimes K}(\sX)=\sP_2(\sX_1)\times\dots\times \sP_2(\sX_K)$ are probability distributions on the product space which have finite second moments and factorize, we define 
\begin{equation*}
W_{2,L}^2(q,\tilde q) = \sum_{k=1}^K L_k W^2_2(q_k, \tilde q_k).
\end{equation*}
The geodesic $(q_t)_{t \in [0,1]}$ between $q$ and $\tilde q$ is simply the product of the geodesics in each factor. Assuming all factors of $q$ have a density, with $T_{q_k}^{\tilde q_k}$ the optimal transport map between $q_k$ and $\tilde q_k$, the $k$-th component of $q_t$ is $((1-t)\mathrm{Id} + t T^{\tilde q_k}_{q_k})_\# q_k$. Equivalently, with $T_q^{\tilde q}(x) = (T_{q_1}^{\tilde q_1}(x_1), \ldots, T_{q_K}^{\tilde q_K}(x_K))$ which we see as map from $\sX$ to $\sX$, it reads $q_t = ((1-t)\mathrm{Id} + t T^{\tilde q}_{q})_\# q$. Note that such $q_t$ belongs to $\sP_2^{\otimes K}(\sX)$ for any $t \in [0,1]$.
It is easy to check that~\eqref{eq:def_geo_factor} generalizes to: for any $t,s$
\begin{equation}
\label{eq:def_geo}
W_{2,L}(q_t,q_s) = |t-s| W_{2,L}(q,\tilde q).
\end{equation}
The entropy behaves well in the Wasserstein geometry, this is~\eqref{eq:convexity_entropy}. This is true as well for the potential energy $f(q) = \int U d q$. If $q$, $\tilde q$ are in $\sP_2^{\otimes K}(\sX)$, then 
\begin{equation}
\label{eq:f_convex_gradient}
f(\tilde{ q}) \geq f(q) + \int_{\sX}\nabla U(x)^\top (T_{ q}^{\tilde{q}}(x)-x) q(dx) + \frac{\lambda^*}{2} W_{2,L}^2( q,\tilde{ q})\,.
\end{equation}
The inequality in \eqref{eq:f_convex_gradient} could be deduced from the general theory of convex functionals on the Wasserstein space~\cite{AGS2008}, but note that here the proof is elementary: we can simply start from~\eqref{eq:convexity_U_gradient}, evaluate it for $y = T_q^{\tilde q}(x)$, and integrate it in $q(dx)$. The key point is that $\int U(T_q^{\tilde q} (x)) q(dx) = \int U(y) \tilde q(dy) = f(\tilde q)$ by the change of variables formula for the pushforward operator, and this yields~\eqref{eq:f_convex_gradient}. Using the same technique for the equation~\eqref{eq:smoothness_U} we can also deduce for any $k \in \{ 1, \ldots, K \}$
\begin{equation}
\label{eq:f_smooth_gradient}
f(\tilde{ q}_k, q_{-k}) \leq f(q) + \int_{\sX_k} \nabla_k U(x)^\top (T_{ q_k}^{\tilde{q}_k}(x_k)-x_k) q(dx_k) + \frac{L_k}{2} W_{2}^2( q_k,\tilde{ q}_k)\,.
\end{equation}
Finally, calling $(q_t)_{t \in [0,1]}$ the geodesic between $q$ and $\tilde q$, with the same method we see that~\eqref{eq:asmp_convex_u_classical} yields
\begin{equation}
\label{eq:f_str_convex_classical}
f(q_t) \leq (1-t)f(q) + t f(\tilde q) - \frac{\lambda^*}{2} t(1-t) W_{2, L}^2(q,\tilde q). 
\end{equation}

\section{Proof of the main result}
\label{sec:proof}

We now move on to the proof of Theorem~\ref{thm:RS_CAVI_contraction}. 
The proof studies how $F(q^{(n+1)})$ relates to $F(q^{(n)})$, and then concludes by induction.

We first show that $\mathbb{E}[F( q^{(n+1)})| q^{(n)}]$ is connected to the Moreau-Yosida regularization of the function $F$. At this level there is no difference between the case $\lambda^* > 0$ and $\lambda^* = 0$.

\begin{lemma}\label{lm:bound_exp_with_H}
Under Assumption \ref{ass:conv_smooth_potential} and if $q^{(n)} \in \sP_2^{\otimes K}(\sX)$, we have for any $n \geq 0$
\begin{align}\label{eq:bound_exp_prox}
\mathbb{E}[F( q^{(n+1)})| q^{(n)}]
\leq\frac{K-1}{K}F( q^{(n)})+\frac{1}{K}Q_{1-\lambda^*}[F](q^{(n)}) \, ,
\end{align}    
where, for any $t\geq 0$, $Q_t[F]$ is the Moreau-Yosida regularization defined by
\begin{align*}
Q_t[F](q)
=
\min_{\tilde{ q}\in\sP^{\otimes K}_2(\sX)}
\left(
F(\tilde{ q})
+
\frac{t}{2}W_{2,L}^2( q,\tilde{ q})
\right).
\end{align*}  
\end{lemma}

We refer to~\cite[Section 3.1]{AGS2008} and references therein for a thorough study of Moreau-Yosida regularization in metric spaces. 
Our definition is slightly changed for the sake of clarity, as one usually uses $1/t$ as the parameter to define the strength of the regularization. The proof of the Lemma is inspired by the analysis of coordinate gradient descent in Euclidean spaces, in particular by \cite[Lemma 2]{RT14} and \cite[Theorem 4]{W15}.

\begin{proof}
Let $\tilde{q}$ be any element of $\sP_2^{\otimes K}(\sX)$. We denote by $T_k$ the optimal transport on $\sX_k$ between $q^{(n)}_k$ and $\tilde q_k$. Let us also write $H(q_1, \ldots, q_K) = \sum_{k=1}^K H(q_k)$ for the total entropy of the probability distribution $q$. 
By definition of Algorithm \ref{alg:rs-cavi} and using the smoothness assumption \eqref{eq:f_smooth_gradient}, we have
\begin{align*}
F ( q^{(n+1)}) 
& \leq 
F(\tilde{q}_{k_n},q^{(n)}_{-k_n})\\ 
& = f(\tilde{q}_{k_n},q^{(n)}_{-k_n}) + 
H(\tilde{q}_{k_n},q^{(n)}_{-k_n})  \\
&\leq
f(q^{(n)}) +
\int_{\sX_{k_n}} \nabla_{k_n} U(x)^\top(T_{k_n}(x_{k_n})-x_{k_n}) q^{(n)}_{k_n}(dx_{k_n}) \\
& \phantom{\leq f(q^{(n)})} +\frac{L_{k_n}}{2} W_2^2(q^{(n)}_{k_n},\tilde{q}_{k_n}) + 
H(\tilde{q}_{k_n},q^{(n)}_{-k_n}).
\end{align*}
Taking the expectation with respect to $k_n$ and using $\sum_k H(\tilde{q}_{k},q^{(n)}_{-k}) = H(\tilde q)+ (K-1) H(q^{(n)})$ as $H$ decomposes additively, we have
\begin{align*}
& \mathbb{E}  [F( q^{(n+1)})| q^{(n)}] \\
& \leq
\frac{1}{K}\sum_{k=1}^K
\Bigg( f(q^{(n)}) +
\int_{\sX_k}\nabla_k U(x)^\top(T_k(x_k)-x_k) q^{(n)}_k(dx_k) \\
& \qquad \qquad +
\frac{L_k}{2} W_2^2(q^{(n)}_k,\tilde{q}_k) + 
H(\tilde{q}_{k},q^{(n)}_{-k})
\Bigg)  \\
& = 
\frac{K-1}{K} F(q^{(n)}) + \frac{1}{K}\Bigg( f(q^{(n)}) \\
&\qquad + \sum_{k=1}^K \int_{\sX_k}\nabla_k U(x)^\top(T_k(x_k)-x_k) q^{(n)}_k(dx_k) + 
\frac{1}{2}W_{2,L}^2(q^{(n)},\tilde{q}) 
+ 
H(\tilde q)  \Bigg).
\end{align*}
At this point we use the convexity of $f$ in the form~\eqref{eq:f_convex_gradient} to obtain
\begin{align*}
f(q^{(n)}) + \sum_{k=1}^K \int_{\sX_k}\nabla_k U(x)^\top(T_k(x_k)-x_k) q^{(n)}_k(dx_k) \leq f(\tilde{q}) - \frac{\lambda^*}{2} W_{2,L}^2(q^{(n)},\tilde{q}) .
\end{align*}
Combining the above inequalities we obtain the bound: 
\begin{align*}
\mathbb{E} & [F( q^{(n+1)})| q^{(n)}] \leq 
\frac{K-1}{K} F(q^{(n)}) + \frac{1}{K} \left( F(\tilde{q}) + \frac{1- \lambda^*}{2}W_{2,L}^2(q^{(n)},\tilde{q})\right).
\end{align*}
We obtain our conclusion by minimizing over $\tilde q \in \sP_2^{\otimes K}(\sX)$ which was arbitrary so far.
\end{proof}

It is easy to check that $F(q^*) \leq Q_{1-\lambda^*} [F](q) \leq F(q)$, but the key point to obtain a convergence rate is to control quantitatively how much $Q_{1-\lambda^*} [F]$ is smaller than $F$. To that extend we rely on convexity of $F$ along geodesics. The lemma below is elementary and, as can be seen from the proof, it holds for \emph{any} function on a metric space which is $\lambda^*$-convex along geodesics.

\begin{lemma}\label{lm:H_contraction}
Let $q\in\sP_2^{\otimes K}(\sX)$. Then under Assumption \ref{ass:conv_smooth_potential} we have for $\lambda^* > 0$
\begin{align}\label{eq:prox_contr_str_conv}
Q_{1-\lambda^*}[F](q)
-F(q^*)
\leq
\left(1-\lambda^*\right)
(F( q)-F(q^*)),
\end{align} 
whereas when $\lambda^* = 0$
\begin{align}\label{eq:prox_contr_conv}
Q_1[F](q)
-F(q^*)
\leq
\left( 1 - \frac{F(q)-F(q^*)}{2R(q)^2}\right) (F(q)-F(q^*)),
\end{align}
with $R(q)=\max\{\sqrt{F(q)-F(q^*)},W_{2,L}( q,q^*)\}$, with $q^*$ being any minimizer of $F$.
\end{lemma}
\begin{proof}
Denote $\bar{F}(q)=F(q)-F(q^{*})$ and $Q_{1-\lambda^*}[\bar{F}](q)=Q_{1-\lambda^*}[F](q)-F(q^{*})$ for brevity, and let $(q_t)_{t \in [0,1]}$ the geodesic joining $q_0 = q$ to $q_1 = q^*$. In particular $\bar{F}(q_1) = \bar{F}(q^*) = 0$. 
Combining the $\lambda^*$ convexity of $f$ (see~\eqref{eq:f_str_convex_classical}) and the fact that $t \mapsto H((q_{t})_k)$ is convex for any $k$ (see~\eqref{eq:convexity_entropy}), we obtain that $\bar{F}$ is $\lambda^*$-convex along geodesics: for any $t \in [0,1]$,
$$
\bar{F}(q_t) 
\leq
(1-t) \bar{F}(q) - \frac{\lambda^*}{2} t (1-t) W_{2,L}^2(q,q^*).
$$
On the other hand, as $q_t$ is a geodesic, \eqref{eq:def_geo} yields $W_{2,L}(q,q_t) = t W_{2,L}(q_1,q_0) = tW_{2,L}(q,q^*)$. Putting these two bounds together we obtain for any $t \in [0,1]$
\begin{align*}
Q_{1-\lambda^*}[\bar{F}](q)
 \leq \bar{F}(q_t) + \frac{1}{2}(1-\lambda^*)W_{2,L}^2( q,q_t)  
 \leq (1-t) \bar{F}(q) + \frac{1}{2} t (t - \lambda^*) W_{2,L}^2(q,q^*).
\end{align*}
If $\lambda^* > 0$ the natural choice is to use $t = \lambda^*$ to cancel the second term and we conclude
$$
Q_{1-\lambda^*}[\bar{F}](q)
\leq
(1-\lambda^*)\bar{F}(q).
$$
In the case $\lambda^*=0$ we use the upper bound $W_{2,L}^2(q,q^*) \leq R(q)^2$ and the value $t_* = \bar{F}(q)/ R(q)^2 \in [0,1]$: we then have
\begin{align*}
&Q_{1}[\bar{F}](q) \leq \bar{F}(q_t) + \frac{W_{2,L}^2( q,q_t)}{2}
\leq
(1-t_*) \bar{F}(q) + \frac{t_*^2 R(q)^2}{2}
=
\left(  1 - \frac{\bar{F}(q)}{2R(q)^2} \right) \bar{F}(q).\\ 
\end{align*}
\end{proof}

\begin{proof}[\textbf{Proof of Theorem \ref{thm:RS_CAVI_contraction}}]
Denote $\bar{F}(q)=F(q)-F(q^{*})$ for brevity and assume, without loss of generality, $\bar{F}(q^{(0)})<\infty$.

\emph{For (a), the strongly convex case}. Note that thanks to~\eqref{eq:control_F_W2} all iterates are in $\sP^{\otimes K}_2(\sX)$. We combine our two preliminary lemmas, specifically~\eqref{eq:bound_exp_prox} and~\eqref{eq:prox_contr_str_conv}, and obtain for any $n \geq 0$
\begin{align*}
\mathbb{E}[\bar{F}(q^{(n+1)})| q^{(n)}]
\leq\frac{K-1}{K}\bar{F}( q^{(n)})+\frac{1}{K}(1-\lambda^*)\bar{F}( q^{(n)})
=\left(1-\frac{\lambda^*}{K}\right)\bar{F}( q^{(n)})\,.
\end{align*}
Taking expectations, by the tower property we obtain the geometric decay of $\mathbb{E}[\bar{F}(q^{(n)})]$ reported in~\eqref{eq:RS_CAVI_contraction}.

\emph{For (b), the convex case}. As $\bar{F}(q^{(n)})$ is decreasing, we have $\bar{F}(q^{(n)})\leq \bar{F}(q^{(0)})\leq R^2$ and $R(q^{(n)})\leq R$ for any $n\geq 0$. As in (a) we start with \eqref{eq:bound_exp_prox} and then use~\eqref{eq:prox_contr_conv} to control the part with the Moreau-Yosida regularization: for any $n \geq 0$
\begin{align*}
\mathbb{E}[\bar{F}( q^{(n+1)})| q^{(n)}]
& \leq\frac{K-1}{K}\bar{F}( q^{(n)})+\frac{1}{K}
\left( 1 - \frac{\bar{F}( q^{(n)})}{2R^2} \right)
\bar{F}( q^{(n)}) \\
& =
\left( 1 - \frac{\bar{F}( q^{(n)})}{2KR^2} \right)
\bar{F}( q^{(n)})\,.
\end{align*}
Taking expectation, by the tower property and Jensen's inequality we have
\begin{align*}
\mathbb{E}[\bar{F}( q^{(n+1)})]
\leq
\mathbb{E}\left[\left( 1 - \frac{\bar{F}( q^{(n)})}{2KR^2} \right)
\bar{F}( q^{(n)})\right]
\leq
\left( 1 - \frac{\mathbb{E}[\bar{F}( q^{(n)})]}{2KR^2} \right)
\mathbb{E}[\bar{F}( q^{(n)})].
\end{align*}
To simplify notation call $u_n = \mathbb{E}[\bar{F}( q^{(n)})]$. Then $(u_n)_{n\geq 0}$ is a non-negative decreasing sequence such that $u_{n+1} \leq (1-u_n/(2KR^2)) u_n$. To find its asymptotic behavior we compute for $m \geq 0$ 
\begin{equation*}
\frac{1}{u_{m+1}} - \frac{1}{u_m} = \frac{u_m - u_{m+1}}{u_{m+1} u_m} \geq \frac{u_m^2/(2KR^2)}{u_{m+1} u_m} = \frac{1}{2KR^2} \cdot \frac{u_{m}}{u_{m+1}} \geq \frac{1}{2KR^2}.
\end{equation*}
We sum this from $m=0$ to $n-1$ and obtain $1/u_n - 1/u_0 \geq n/(2KR^2)$. 
With $u_0 = \bar{F}( q^{(0)}) \leq R^2$ we obtain our conclusion
\begin{align*}
&\frac{1}{u_n} \geq \frac{n}{2 KR^2} + \frac{1}{R^2} = \frac{n+2K}{2KR^2}\,.\\
\end{align*}
\end{proof}

\appendix

\section{Bounding the distance along iterates in the convex case}\label{sec:bounds_on_R}

We report here the arguments showing that the iterates of CAVI remain bounded in quadratic Wasserstein distance, inspired by the analysis of \cite{AL24} but improving the dependence in $F(q^{(0)})$ following the suggestion of an anonymous reviewer.

\begin{proposition}\label{prop:R_bounded}
Let $N \geq K$ the random time where each of the factors has been updated at least once. There exists a finite constant $C$ depending only on $\pi$ such that, for any $n \geq N$, 
\begin{equation*}
\int_{\sX} \|x \|^2 \, q^{(n)}(d x) \leq C \left(F \left(q^{(0)} \right)^2 + 1 \right).
\end{equation*}
\end{proposition}

\begin{proof}
The variational representation of the entropy (see e.g. Equation (3.3) in~\cite{AL24}) yields, for any non-negative function $g$,
\begin{equation*}
\int_\sX g \, d q \leq F(q) + \log \int_{\sX} e^{g} \, d \pi\,.
\end{equation*}
As $\exp(-U)$ is integrable and $U$ is convex, it is easy to show that there exist constants $a > 0$ and $b \in \R$ such that $U(x) \geq a \| x \| + b$. Thus when $g(x) = \frac{a}{2} \| x \|$, with $C_1 = \log \int_{\sX} e^{g} \, d \pi < + \infty$ which depends only on $\pi$, we have the first moment bound
\begin{equation*}
\int_\sX \| x \| \, q(dx)  \leq \frac{2}{a} F(q) + \frac{2C_1}{a}. 
\end{equation*}
From~\eqref{eq:cavi_alg_update} it is easy to see that $U^{(n)}_{k}$ is convex if the $k$-th factor has been updated at least once, so that the distribution $q^{(n)}$ is log-concave for $n \geq N$. Using the moment inequality 
\begin{equation*}
\int_{\sX} \|x \|^2 \, q(dx) \leq C_2 \left( \int_{\sX} \| x \| \, q(dx)  \right)^2\,,
\end{equation*}
which is valid for any log-concave distribution $q$  for a universal constant $C_2$  \cite[Equation (B.4)]{BobkovLedoux2019}, we obtain for $n \geq N$
\begin{equation*}
\int_{\sX} \|x \|^2 \, q^{(n)}(dx) \leq C_2 \left( \int_{\sX} \| x \| \, q^{(n)}(dx)  \right)^2 \leq C_2 \left( \frac{2}{a} F\left(q^{(n)}\right) + \frac{2 C_1}{a}  \right)^2. 
\end{equation*}
As in addition $F\left(q^{(n)} \right)$ decreases with $n$ we reach our conclusion after renaming the constants.
\end{proof}

\begin{corollary}
\label{corollary:bound_distance_entropy}
Let $N \geq K$ the random time where each of the factors has been updated at least once. There exists a constant $C$ depending only on $\pi$ such that, for any $n \geq N$, 
\begin{equation*}
W_{2,L}(q^{(n)},\pi) \leq C \left(F \left(q^{(0)} \right)^2 + 1 \right).
\end{equation*}
\end{corollary}

\begin{proof}
We use the triangle inequality with $\delta_0$ the Dirac mass at $0$, and the inequality $\| x \|_L^2 \leq L_\text{max} \| x \|^2$:
\begin{align*}
W_{2,L}(q^{(n)},\pi) & \leq W_{2,L}(q^{(n)},\delta_0) + W_{2,L}(\delta_0,\pi) \\
& \leq  L_\text{max}^{1/2} \sqrt{\int_{\sX} \|x \|^2 \, q^{(n)}(d x)} +  L_\text{max}^{1/2} \sqrt{\int_{\sX} \|x \|^2 \, \pi(d x)} \\
& \leq L_\text{max}^{1/2} C_1  \left(F \left(q^{(0)} \right)^2 + 1 \right) + L_\text{max}^{1/2} C_2, 
\end{align*}
where the constant $C_1, C_2$ depend only on $\pi$. We obtain the result by an appropriate definition of $C$.
\end{proof}

\subsection*{Acknowledgments}

We warmly thank an anonymous reviewer for suggesting a substantial improvement of the proof of Proposition~\ref{prop:R_bounded}.

\bibliographystyle{siamplain}
\bibliography{bibliography}

\end{document}